\documentclass[letterpaper, 10 pt, conference]{ieeeconf}  

\IEEEoverridecommandlockouts                              

\usepackage{microtype}
\usepackage{graphicx}
\usepackage{subfigure}
\usepackage{booktabs} 
\usepackage{algorithm}
\usepackage{algorithmic}
\usepackage[hidelinks]{hyperref}
\usepackage{url}
\usepackage{dsfont} 
\usepackage{siunitx}
\usepackage{amsmath,amssymb,amsfonts}
\usepackage{amsthm}
\usepackage{cleveref}
\usepackage{mathrsfs}
\usepackage{hyperref}
\usepackage{tikz}
 \newtheorem{problem}{Problem}
 \newtheorem{theorem}{Theorem}
 \newtheorem{lemma}{Lemma}
 \newtheorem{corollary}{Corollary}
\usepackage{algorithm}
\usepackage{algorithmic}

\DeclareMathOperator*{\argmax}{arg\,max}

\title{\LARGE \bf
Risk-Averse Multi-Armed Bandits with Unobserved Confounders: A Case Study in Emotion Regulation in Mobile Health
}

\author{Yi Shen, Jessilyn Dunn and Michael M. Zavlanos
\thanks{*This work is supported in part by AFOSR under award \#FA9550-19-1-0169
and by NSF under award CNS-1932011.}
\thanks{Yi Shen and Michael M. Zavlanos are with the Department of Mechanical Engineering and Materials Science, Duke University, Durham, NC, USA. Email: \{yi.shen478, michael.zavlanos\}@duke.edu}%
\thanks{Jessilyn Dunn is with the Department of Biomedical Engineering, Duke University, Durham, NC, USA. Email: jessilyn.dunn@duke.edu}
}

\begin{document}

\maketitle
\thispagestyle{empty}
\pagestyle{empty}

\begin{abstract}
In this paper, we consider a risk-averse multi-armed bandit (MAB) problem where the goal is to learn a policy that minimizes the risk of low expected return, as opposed to maximizing the expected return itself, which is the objective in the usual approach to risk-neutral MAB. Specifically, we formulate this problem as a transfer learning problem between an expert and a learner agent in the presence of contexts that are only observable by the expert but not by the learner. Thus, such contexts are unobserved confounders (UCs) from the learner's perspective. 
Given a dataset generated by the expert that excludes the UCs, the goal for the learner is to identify the true minimum-risk arm with fewer online learning steps, while avoiding possible biased decisions due to the presence of UCs in the expert's data.
To achieve this, we first formulate a mixed-integer linear program that uses the expert data to obtain causal bounds on the Conditional Value at Risk (CVaR) of the true return for all possible UCs. We then transfer these causal bounds to the learner by formulating a causal bound constrained Upper Confidence Bound (UCB) algorithm to reduce the variance of online exploration and, as a result, identify the true minimum-risk arm faster, with fewer new samples. We provide a regret analysis of our proposed method and show that it can achieve zero or constant regret. Finally, we use an emotion regulation in mobile health example to show that our proposed method outperforms risk-averse MAB methods without causal bounds.

\end{abstract}

\section{Introduction}
\label{sec:0_intro}
Multi-armed bandit (MAB) problems are sequential decision making problems where an agent sequentially selects arms to pull and receives a random reward in order to learn the reward distributions of all arms and at the same time to find a strategy that maximizes the total expected reward. Many applications, ranging from treatment design \cite{murphy2003optimal} and news article recommendation \cite{li2010contextual} to online marketing \cite{misra2019dynamic}, can be formulated as MAB problems. However, risk-neutral formulations that only maximize the total expected reward do not always provide desirable solutions. For example, in mobile health-based interventions for emotion regulation (ER) \cite{ameko2020offline}, the strategy with the highest average effectiveness rate (i.e., generating the highest possible positive emotions) is not necessarily the best; minimal adverse reactions (e.g., behaviors that cause harm to self or others) are also necessary. To avoid rare but catastrophic outcomes, appropriate risk-averse criteria can be considered during learning, e.g., Conditional Value at Risk (CVaR) \cite{artzner1999coherent}. 

Oftentimes, in risk-sensitive applications, such as in the mobile health interventions for ER discussed above, the data collected by experts are available but miss important confounding variables that influence both the dependent and independent variables, causing spurious association effects. This is, e.g., the case when mobile health devices can only collect partial data due to technological limitations and/or privacy concerns. In this situation, when the expert observational data contain contexts that are unobserved confounders (UCs) to the learner, it is well known that the average treatment effects cannot be estimated without bias, regardless of the sample size \cite{pearl2009causality}. Instead, what can be computed using ideas from causal inference are causal bounds on the true treatment effects that include all possible UC realizations, as shown in the seminal work by \cite{balke1997bounds}.

Causal bounds computed from expert observational data have been recently used to develop transfer learning methods for bandit problems \cite{zhang2017transfer}, reinforcement learning problems \cite{zhang2020transfer} and imitation learning problems \cite{liu2021learning}. Specifically, \cite{zhang2017transfer} propose a transfer learning method that uses causal bounds computed from expert data that contain UCs to obtain a causal bound constrained Upper Confidence Bound (UCB) algorithm that the learner can use to learn an optimal policy with few new data samples. More recently, \cite{zhang2020transfer} extend this framework to reinforcement learning problems by computing causal bounds on value functions and using these bounds to develop a causal bound constrained Q-learning algorithm for the learner. Motivated by these approaches, \cite{liu2021learning} consider an imitation learning problem, where the expert's policies can be modeled as different arms in a MAB and the learner's goal is to learn the best arm, i.e., policy, and improve it online using only a few new data samples. Common in the above methods is that they have been designed for and are only applicable to risk-neutral problems. 

Risk-sensitive bandit problems are studied in \cite{sani2013risk,vakili2016risk,cassel2018general,tamkin2019distributionally}. Specifically, \cite{sani2013risk,vakili2016risk} extend classic risk-neutral MAB to risk-averse MAB using a mean-variance risk measure. Since in risk-averse MAB the optimal policy is not necessarily a single-arm policy as in risk-neutral MAB, the authors bound the performance gap between the optimal policy and the optimal single arm policy and incorporate this gap into the regret analysis. On the other hand, \cite{cassel2018general} provide a systematic approach for regret analysis in MAB under different risk criteria such as value-at-risk, Conditional Value at Risk (CVaR), and Sharpe-ratio. Finally, \cite{tamkin2019distributionally} present a distributionally-aware method that adds an exploration term to the estimated cumulative distribution function (CDF) of the rewards and achieves better empirical results  compared to \cite{cassel2018general} under the CVaR risk measure. Common in all these methods is that they merely focus on online learning problems and do not rely on any existing rich observational data. 

In this work, we propose a new transfer learning method for risk-averse MAB that can handle UCs in the expert data. Specifically, we consider an expert and a learner agent, both modeled as contextual MAB \cite{langford2007epoch,slivkins2011contextual}, and assume that the expert is presented with additional contextual information (e.g., a person's past activities and locations) that can affect both the expert's policy, or the selection of the arms (e.g., the next ER intervention to implement), and the reward function. This contextual information is not recorded in the expert dataset that is provided to the learner (e.g., a person's mobile health device) and, therefore, it constitutes an UC for the learner. 
The goal of the learner is to identify the optimal arm with fewer online interactions with the environment. To do so, we first formulate a mixed-integer linear program (MIP) that utilizes the observational data to calculate causal bounds on the CVaR values of the true reward function. We then transfer these causal bounds to the learner and propose a causal bound constrained UCB algorithm to avoid risky online exploration and learn the optimal arm without accruing bias from the observational data. We provide a regret analysis that shows that it is possible to achieve zero or constant regret using causal bounds. Finally, we illustrate our proposed method on the mobile health example,  which aims to optimize an intervention for ER to increase positive emotions and decrease negative emotions, while avoiding high-risk behaviors that could cause harm to self or others.

To the best of our knowledge, transfer learning methods for risk-averse MAB have not been studied in the literature. Perhaps the most closely related works to the proposed method are \cite{zhang2017transfer,tamkin2019distributionally}. Compared to \cite{zhang2017transfer}, the calculation of causal bounds proposed here can handle additional assumptions on UCs and, as a result, can return tighter causal bounds. Moreover, the analysis in \cite{zhang2017transfer} is tailored to risk-neutral bandits and cannot be directly adapted to risk-averse problems. Compared to \cite{tamkin2019distributionally}, here we use the same risk measure but also leverage rich observational data that are available in many risk-averse applications. As a result, we can obtain better online performance even in the presence of UCs.     

The rest of the paper is organized as follows. In section \ref{sec:1_prelim}, we discuss the models of the expert and the learner and define the transfer learning problem. In section \ref{sec:2_method}, we formulate the optimization problem to compute the causal bounds on CVaR values using observational data and present the proposed causal bound constrained UCB algorithm. In Section \ref{sec:3_regret}, we present regret analysis results and show that causal bounds can achieve lower regret under certain conditions. In Section \ref{sec:4_results}, we present numerical results on a mobile health-based ER intervention application to illustrate a real-world example application, as well as the effectiveness of the proposed method. Finally, we conclude this work in Section \ref{sec:5_conc}.

\section{Problem Definition}
\label{sec:1_prelim}
Consider a contextual MAB problem defined by the tuple $\left(C,X,Y^C(X)\right)$, where $C$ is a random variable that models the context, $X\in\{1,\dots,K\}$ is a random variable that indicates the selection of one of $K$ arms, and $Y^c(x)$ is a random reward function associated with arm $X=x$ given context $C=c$. At each time step, a sample context $c$ is drawn independently from a distribution $P(C)$ and is announced to the agent. Then, an agent chooses one of the arms and the reward associated with this arm is revealed. 
%
Without loss of generality, we assume the rewards are non-negative and upper bounded by $U\in\mathbb{R}$. 
For example, in ER mobile health applications, $c$ can represent the user's demographic information and/or previous activities and $X$ can be the set of all possible treatments to relief anxiety. Then, the psychological clinician that is the expert agent can prescribe one of the possible treatments and observe the outcome $y^c(x)$, i.e., the effectiveness of emotion regulation after the treatment. The goal is to find a context-dependent policy that achieves the best outcome.   

Given a contextual MAB problem, we can define a standard MAB problem induced by it as $\left(X,\mathbb{E}_{C}[Y^C(X)]\right)$, where $X$ is a random variable indicating the selection of an arm as in contextual MAB and $\mathbb{E}_{C}[Y^C(X)]$ is the expected random reward function with respect to the distribution $P(C)$.
We then model the learner's decision process as a standard MAB since the context is not observable to the learner. For example, in the same mobile health application discussed above, due to privacy concerns, the clinician that is the learner agent may not have access to obtain the patients' demographic information; owing to sensor limitations, cheap wearable devices cannot measure the same number of contexts as more expensive ones (the expert agent). In this case, the context $C$ is a random variable that is unobserved by the learner, yet it affects the outcome of the possible treatments. The goal of the learner is to find a context-independent policy that achieves the best outcome. Specifically, we are interested in the case where the learner's performance is evaluated by a risk-averse measure.

In this paper, we use the CVaR as the risk measure for the learner as in \cite{tamkin2019distributionally}. CVaR measures the expected value of a distribution's tail. Formally, let $X$ be a bounded random variable with CDF $F_X(x)=P(X\leq x).$ The CVaR at level  $\alpha\in(0,1)$ of the random variable $X$ is defined as \cite{rockafellar2000optimization}:
\begin{align*}
    \text{CVaR}_{\alpha}(X)=\sup_{\nu}\{\nu -\frac{1}{\alpha}\mathbb{E}[(\nu-X)^+]\},
\end{align*}
where $[x]^+=\max\{x,0\}$. We sometimes write $\text{CVaR}_{\alpha}(F_X)$ for $\text{CVaR}_{\alpha}(X)$, where $F_X$ is the CDF of the random variable $X$.
Given the conditional value at risk level $\alpha\in(0,1]$, we define the CVaR regret associated with a MAB at time $n$, the same as in \cite{tamkin2019distributionally}, as
\begin{align}\label{eq:cvarregret}
    R_n^{\alpha} = n \max_{x\in [K]}\left(\text{CVaR}_{\alpha}(F_{x})\right) - \mathbb{E}\left[\sum_{t=1}^n\text{CVaR}_{\alpha}(F_{X_t})\right],
\end{align}
where $X_t$ is the action taken at time step $t$, $[K] =\{1,\ldots,K\}$ and $F_x$ is the CDF of the distribution of rewards of the arm $x$. The goal of the learner is to find the arm $x$ that minimizes the CVaR regret.

Assume now a dataset $\tau_E=\{(x_t,y_t)\}_{t=1}^N$ generated by the expert, where $x_t,y_t$ are the action taken and reward received at time step $t$. Note that the expert's actions depend on the contextual information $c$, but this information is not recorded in $\tau_E$. Then, in this paper, we address to solve the following problem.

\begin{problem}\label{problem1}
Given the observational data $\tau_E$ generated by an expert (modeled by a contextual MAB), design a transfer learning algorithm for the learner (modeled by the induced MAB) that leverages the data $\tau_E$ to find an optimal arm selection strategy that minimizes the CVaR regret defined in \eqref{eq:cvarregret}. 
\end{problem}

Since action-reward pairs are given in the observation data, one might attempt to first calculate the $\text{CVaR}_{\alpha}$ value for each arm and transfer the arm with the highest $\text{CVaR}_{\alpha}$ value to the learner. It turns out that with UCs this approach can return an arm that is sub-optimal or even the worst. See Example 1 in \cite{liu2021learning} for a case where this naive transfer does not work for risk-neural bandits, that are a special case of CVaR bandits for $\alpha=1$.  

\section{Transfer Learning with Unobserved Confounders}\label{sec:2_method}
In this section, we propose a TL framework to solve Problem \ref{problem1}. We first formulate an optimization problem that calculates causal bounds on CVaR. We then transfer the causal bounds to the learner and propose a causal bound constrained risk-averse MAB algorithm.    
\subsection{Causal Bound Optimization}
\begin{figure}
	\centering
	\subfigure[\footnotesize Expert]{\label{subfig:demostrator}\includegraphics[scale=0.06]{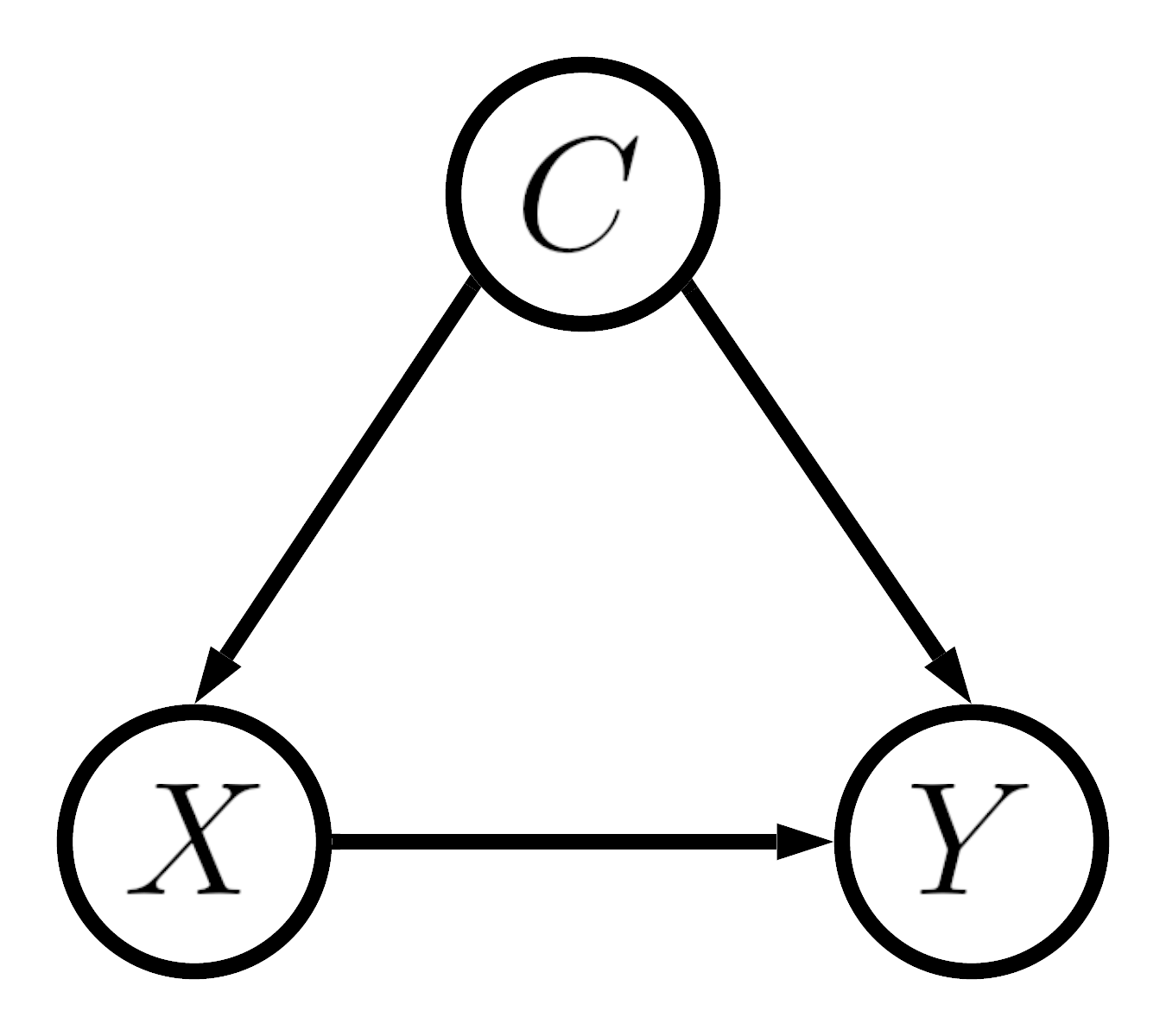}}\qquad \qquad
	\subfigure[\footnotesize Learner]{\label{subfig:learner}\includegraphics[scale=0.06]{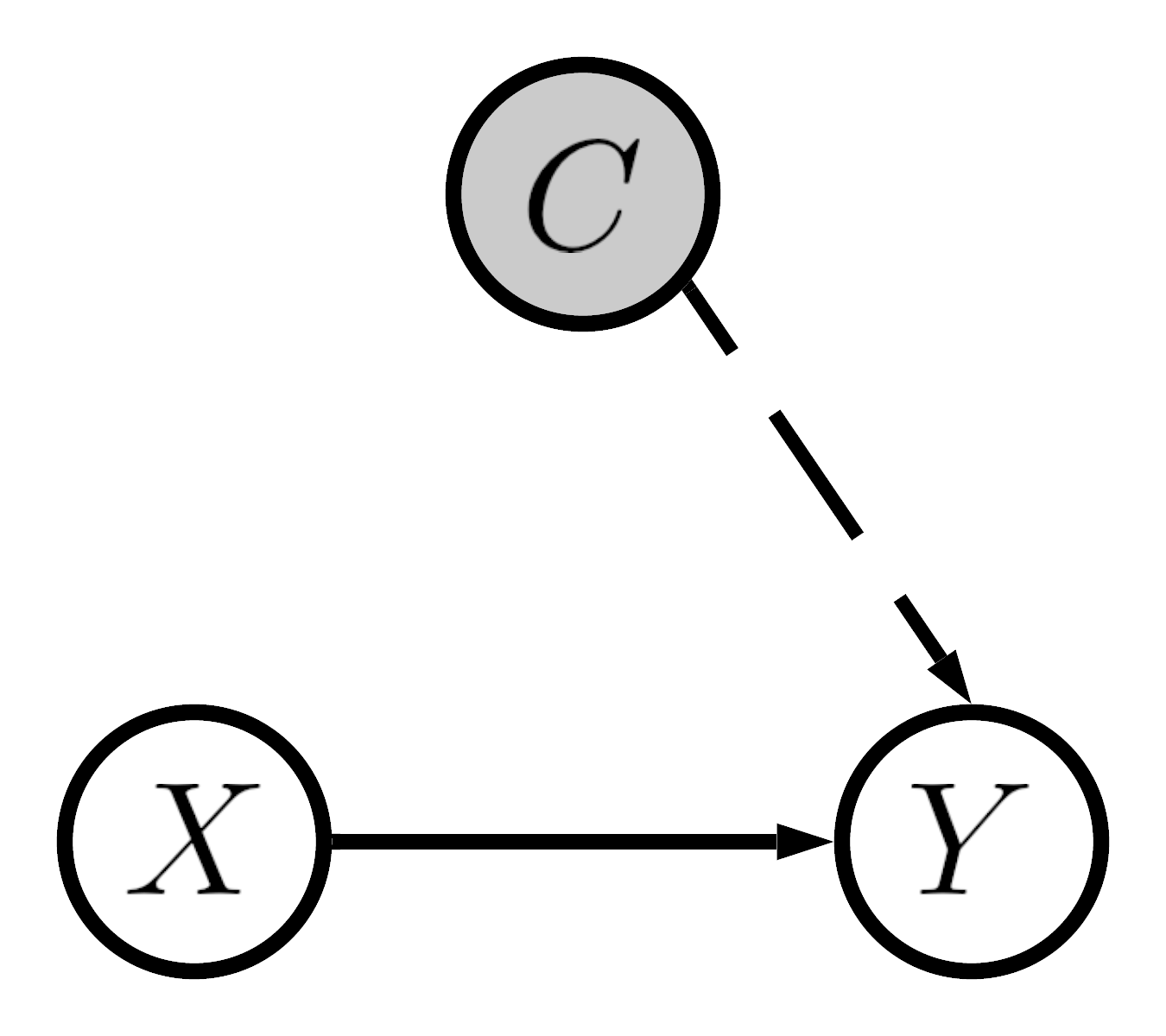}}
	\caption{\small Expert and learner causal graphs. (\textit{a}) the causal graph of expert; (\textit{b}) the causal graph of the learner. The gray nodes represent the unobserved confounder $C$ in the learner's model, and the white nodes represent the observable data. The arrows indicate the directions of the causal effects.}
	\label{figure:causal graph}
	\vspace{-4mm}
\end{figure}

The MAB models introduced in Section \ref{sec:1_prelim} can be seen through a causal lens as depicted in Figure \ref{figure:causal graph}. Specifically, the learner's model can be thought of as a special case of the expert's model equipped with the $do(\cdot)$ operation. The intervention (action) $do(X=x)$ represents a model manipulation where the value of $X$ is set to $x$ regardless of how it is originally determined in the model, as discussed in \cite{pearl2009causality} Chapter 3. Defining the action in the expert's model by the $do(X=x)$ operator recovers the learner's model as the action $x$ is selected regardless of contextual information. Note that the reward functions still depend on the contexts. With the above definitions, Problem \ref{problem1} reduces to finding values of $P(y|do(X=x))$ for all action and reward pairs given observational data $\tau_E=\{(x_t,y_t)\}_{t=1}^N$ that do not contain contextual information. In what follows, we will denote $do(X=x)$ by $do(x)$. We assume that the contexts $C$, the actions $X$, and the rewards $Y$ take discrete values. 

It is well-known that in general $P(y|do(x)) \neq P(y|x)$ if UCs exist \cite{pearl2009causality}. \cite{tian2000probabilities} provide a simple formula for bounding $P(y|do(x))$ in the presence of UCs, that is 
\begin{align}\label{eq:simplebound}
    P(x,y) \leq P(y|do(x)) <1- P(x,y'),
\end{align}
where $y'$ represents a set of the values that are not equal to $y$. Note that the probabilities $P(x,y)$ for any action reward pairs can be estimated from the observation data even though $P(y|do(x))$ cannot.
%
%
However, the obtained causal bounds according to \eqref{eq:simplebound} are usually loose since $1- P(x,y')-P(x,y)$ is often close to 1. This limitation can be overcome if additional information is used to develop these bounds that may be available in practice. 
For example, even though the person's exact current activity such as walking or sitting as a contextual information for ER design may be unknown, the general distribution of possible activities may be known.
This information may allow to estimate the UCs' distribution $P(C)$. As suggested in \cite{li2021bounds}, given the distribution of UCs $P(C)$, we can find tighter causal bounds compared to those provided by (2) by formulating an appropriate optimization problem. Note that knowing $P(C)$ does not imply knowing the relationship between the contexts, actions and rewards, i.e., knowing the distributions of the lab results of all patients does not imply a single patient's treatment outcome. Using the back-door criterion as in \cite{pearl2009causality}, similar to in \cite{li2021bounds}, we can express $P(y|do(x))$ as
\begin{align}\label{eq:complexbound}
    P(y|do(x))=\sum_c \frac{P(x,y,c)P(c)}{P(x,c)}.
\end{align}
Note that \eqref{eq:complexbound} is not computable since $P(x,y,c)$ and $P(x,c)$ cannot be estimated from the observational data. Nevertheless, we can reformulate \eqref{eq:complexbound} as an optimization problem, as shown in the following result.
\begin{theorem}\label{thm:1}
Consider the causal diagram G in Figure \ref{figure:causal graph} (a). Given the probabilities $P(X,Y)$ and $P(C)$, the causal effects $P(Y|do(X))$ of $X$ on $Y$ are bounded as:
\begin{align}\label{eq:causalboundsop}
    LB(y,x) \leq P(y|do(x)) \leq UB(y,x) \; \forall x\in [K], y\in Y,   
\end{align}
where $LB(y,x)$ and $UB(y,x)$ are the solutions to the optimization problem in \eqref{eq:optimization}.
\begin{align}
    LB(UB)&(y,x)  = \min_{a_c,b_c}(\max_{a_c,b_c}) \quad  \sum_c \frac{a_cP(c)}{b_c} \label{eq:optimization}\\
\text{s.t.}  \quad  & P(c)\geq b_c, \text{ }b_c\geq a_c, \nonumber\\ 
& a_c \leq P(x,y), \text{ } b_c \leq P(x), \nonumber \\
& a_c \geq P(x,y)+P(c)-1, \text{ }b_c \geq P(x)+P(c)-1, \nonumber\\
& a_c, b_c \geq 0, \text{ for all $c\in C$}; \nonumber\\
& \sum_c a_c = P(x,y), \text{ } \sum_c b_c = P(x) \nonumber.
\end{align}
\end{theorem}

Note that the optimization problem is well defined as long as $b_c$ is positive for all $c\in C$. Indeed, it is a linear-fractional optimization problem and can be rewritten as a linear programming problem; see Chapter 4.3 in \cite{boyd2004convex} for details. Theorem \ref{thm:1} is a simple modification of Theorem 4 in \cite{li2021bounds}, therefore, its proof is omitted.
Theorem \ref{thm:1} provides a way to bound the learner's reward probabilities using the expert's data without introducing any biases caused by the UCs. Recall that the learner's goal is to find the best risk-averse arm that minimizes the CVaR regret defined in \eqref{eq:cvarregret}. Using the bounds on $P(Y|do(X))$ developed in \eqref{eq:optimization}, we can calculate the bounds on $\text{CVaR}_{\alpha}(Y|do(X))$. 
\begin{theorem}\label{theorem2}
Assume $Y\in\{y_0,\ldots, y_n\}$ and $y_0 < \ldots < y_n$. Let $a_i,b_i$ the bounds for $P(y_i|do(x))$ obtained by \eqref{eq:optimization} for action $x$, i.e., $a_i\leq P(y_i|do(x))\leq b_i$. Then, for a given risk level $\alpha$, the causal bounds on the CVaR can be obtained by the solution of the optimization problem
\vspace{-6mm}
\begin{align}\label{eq:cvaropgeneral}
    & \text{CVaR}_{\alpha}(Y|do(x))_{\min(\max)} = \min(\max) \quad m \\
    & \text{s.t.} \quad a_i\leq P(y_i|do(x)) \leq b_i \quad \forall i,\nonumber \\
    & \quad \quad \sum_{i} P(y_i|do(x))=1,\text{ where} \nonumber\\
   &  m = \left \{
  \begin{aligned} \label{eq:condtionalconstraint}
    &y_0, \text{ if }\ p(y_0|do(x))\geq\alpha; \\
    & \ldots \\
    & f(y_0,\ldots,y_k,y_{k+1}), \text{ if}\ \sum_{i=0}^{k}P(y_i|do(x))<\alpha,\\
    & \text{and }\sum_{i=0}^{k+1}P(y_i|do(x))\geq\alpha; \\
    & \ldots 
  \end{aligned}, \right. \\
  & \text{and}\ f(y_0,\ldots,y_k,y_{k+1}) = 1/\alpha \cdot \sum_{i=1}^k y_iP(y_i|do(x))\nonumber\\
  &+1/\alpha \cdot y_{k+1}(\alpha-\sum_{i=0}^{k}P(y_i|do(x))).\nonumber
\end{align}
\end{theorem}
\vspace{-8mm}
\begin{proof}
The result follows from CVaR calculations provided in \cite{rockafellar2000optimization,rockafellar2002conditional}. 
\end{proof}
\vspace{-10mm}
Compared to CVaR calculations as in \cite{rockafellar2000optimization,rockafellar2002conditional},  $P(y_i|do(X))$ is not a fixed number here; thus, we formulate a constrained optimization problem to calculate the CVaR values. The conditional constraints in \eqref{eq:condtionalconstraint} can be reformulated using binary variables to indicate whether the conditions hold or not; see e.g., \cite{bradley1977applied} Chapter 9. Thus, the optimization in Theorem \ref{theorem2} is indeed a mixed-integer program problem. Next, we present the exact form of this problem when Y is a binary variable.
\begin{corollary}
Let $Y$ be binary and denote $Y=1$ and $Y=0$ by $y$ and $\Bar{y}$. With the same assumptions as in Theorem \ref{theorem2}, the causal bounds on the CVaR can be obtained by the solution of the mixed-integer program problem
\vspace{-10mm}
\begin{align}
  \text{CVaR}_{\alpha}&(Y|do(x))_{\min(\max)} = \min(\max) \quad  m-n \label{eq:cvarop}\\
    \text{s.t.} \quad  & P(\bar{y}|do(x))\leq \alpha +M(1-m),  \nonumber\\
    &-P(\bar{y}|do(x))\leq -\alpha +Mm, \nonumber\\ 
    & a_0 \leq P(\bar{y}|do(x)) \leq b_0, \; a_1 \leq P(y|do(x)) \leq b_1 ,\nonumber\\
    & n \leq Mm, n \leq P(\bar{y}|do(x))/\alpha, \nonumber\\
    &n \geq P(\bar{y}|do(x))/\alpha-M(1-m) , \nonumber \\
    & P(y|do(x)) + P(\bar{y}|do(x)) = 1, \nonumber\\
    & n \geq 0,  m\in \{0,1\}, \nonumber \\
    & \text{where $M$ is a constant large number.}\nonumber 
\end{align}
 \end{corollary}
\begin{proof}
By Theorem \ref{theorem2}, we know that if the condition $P(\bar{y}|do(x))\geq \alpha$ holds, then the objective function $\text{CVaR}_{\alpha}(Y|do(x))$ is 0; if the condition $P(\bar{y}|do(x))< \alpha$ and $P(\bar{y}|do(x))+P(y|do(x))=1\geq \alpha$ holds, it implies the objective function equals to $(\alpha-P(\bar{y}|do(x)))/\alpha$. We only need to check whether the integer program defined in \eqref{eq:cvarop} has the same objective function as in Theorem \ref{theorem2}. When $m=0$, we have that $-P(\bar{y}|do(x))\leq -\alpha$ and $n=0$, which implies that $m-n=0$; when $m=1$, we have that $-P(\bar{y}|do(x))\leq -\alpha$, $n= P(\bar{y}|do(x))$, which implies that $m-n=1- P(\bar{y}|do(x))/\alpha.$ As a result, the mixed-integer program in \ref{eq:cvarop} is equivalent to the problem in Theorem \ref{theorem2} for the binary outcome case.
 \end{proof}

%
\subsection{Causal Bound Constrained MAB}
\begin{algorithm}[t]
	\begin{small}
		\caption{Causal bound constrained CVaR-UCB}
		\label{alg:cvarucb}
		\begin{algorithmic}[1]
			\REQUIRE{ Risk level $\alpha$, reward range $U$, horizon $n$, CVaR causal lower bound $l_x$ and upper bound $h_x$ for all arms.}
			\STATE{	Remove any arm $x$ with $h_{x}<l_{max}$, where $l_{max} = \max_x \{l_x\}$. Let $[K']$ be the set of remaining arms.}
			\STATE{Choose each arm $x\in [K']$ once.}
			\STATE{Set $\hat{F}_x$ as the empirical CDFs of each arm $x\in [K']$ on $[0,U]$,}
			\STATE{Set $T_x \leftarrow 1.$}
			\FOR{$t=1,\ldots,n$}
			\FOR{each $x\in [K']$}
			\STATE{$\epsilon_x\leftarrow \sqrt{\frac{\ln(2n^2)}{2T_x}}$.}
			\STATE{$\Tilde{F}_x(y)\leftarrow \left(\hat{F}_x(y)-\epsilon_x \mathbb{I}\{y\in[0,U)\}\right)^+$.}
			\STATE{$\text{UCB}_x^{\text{DKWClip}}(t)\leftarrow \tilde{c}_x^{\alpha} := \min\{\text{CVaR}_\alpha(\Tilde{F}_x),h_x\}$.}
			\ENDFOR
			\STATE{Select action $X_t= \argmax_x{\text{UCB}_x^{\text{DKWClip}}}(t)$.}
			\STATE{$T_{X_t}\leftarrow T_{X_t} +1$.}
			\STATE{Update the empirical CDF $\hat{F}_{X_t}$ of arm $X_t$.}
			\ENDFOR
    
		\end{algorithmic}
	\end{small}
	\vspace{0mm}
\end{algorithm}
Using the above causal bounds on CVaR, we propose a causal bound constrained CVaR-UCB algorithm outlined in Algorithm \ref{alg:cvarucb}. Specifically, let $l_x$ and $h_x$ be the lower and upper causal bounds on $\text{CVaR}_{\alpha}(Y|do(x))$ for each action $x\in [K]$. Denote by ${\max}$ the maximum value of all lower bounds, i.e., $l_{\max}=\max_{x\in [K]}l_x$. Similar to CVaR in \cite{tamkin2019distributionally}, we compute the CVaR-UCB for each arm at the beginning of each time step and select the arm with the highest CVaR-UCB. Since the causal bounds on CVaR provide upper bounds on the CVaR values, we can use these causal bounds to clip the CVaR-UCB, denoted as $\text{UCB}_x^{\text{DKWClip}}$, i.e., we take the minimum between CVaR-UCB and $h_x$ for each arm as in step 9 in Algorithm \ref{alg:cvarucb}. These causal constraints can reduce the variance of the CVaR-UCB estimates, thus, they can help avoid pulling sub-optimal arms that have higher CVaR-UCB values. In addition, any arms with $h_x < l_{\max}$ should not be pulled by the learner since the $l_{\max}$ arm is strictly better than them; see in step 1 in Algorithm \ref{alg:cvarucb}. In Algorithm \ref{alg:cvarucb}, we denote by $\hat{F}_x$ the empirical CDF estimate for arm $x$.
%

%

\section{Regret Analysis}\label{sec:3_regret}
In this section, we provide a regret analysis for Algorithm \ref{alg:cvarucb} and show that the proposed algorithm can achieve zero or constant regret under certain conditions. 
\begin{lemma}[Regret Decomposition] \label{lemma}The CVaR regret satisfies the following identity
\begin{align}\label{eq:cvarregret_decompo}
    R_n^{\alpha} = \sum_{x=1}^K\Delta_x^{\alpha}\mathbb{E}[T_x(n)],
\end{align}
where $\Delta_x^{\alpha}=\max_i \text{CVaR}_\alpha(F_i)-\text{CVaR}_\alpha(F_x)$ is the sub-optimality gap of arm x with respect to the optimal CVaR arm and $T_x(n)$ is the number of times arm $x$ has been pulled up to time step $n$.
\end{lemma}
\begin{proof}
The proof follows from Lemma 4.5 in \cite{lattimore2020bandit} by replacing the mean sub-optimality gap with CVaR sub-optimality gap.
\end{proof}
The following result shows that the expected number of times that sub-optimal arms are selected by Algorithm \ref{alg:cvarucb} is no greater than that provided by the CVaR-UCB algorithm in \cite{tamkin2019distributionally}.
\begin{theorem}\label{theorem3}
Let $\mu^*=\max_x \text{CVaR}_\alpha(F_x)$. Then, the expected number of times $\mathbb{E}[T_x(n)]$ that any sub-optimal arm $x$ is pulled by Algorithm \ref{alg:cvarucb} is upper bounded by:
\[\mathbb{E}[T_x(n)] \leq \begin{cases} 
      0 & h_x < l_{\max} \\
      1 & l_{\max} \leq h_x < \mu^* \\
      3+\frac{4\ln(\sqrt{2}n)U^2}{\alpha^2\Delta^{\alpha 2}_x} & h_x\geq \mu^*
   \end{cases}.
\]
\end{theorem}
\begin{proof}
Denote the condition $h_x < l_{\max}$ as $C_1$, $l_{\max} \leq h_x < \mu^*$ as $C_2$ and $h_x\geq \mu^*$ as $C_3$. Note that Algorithm \ref{alg:cvarucb} does not pull any arms that satisfy condition $C_1$ since they are guaranteed to be sub-optimal. Therefore, the expected number of pulls of these arms is 0. Without loss of generality, we assume that the first arm is optimal, i.e., $\mu^* =\text{CVaR}_\alpha(F_1).$
Let $c_x^{\alpha}$ be the CVaR of arm $x$ and $\hat{F}_{x,t}$ denote the empirical CDF of arm $x$ before time step $t$. Then, $\tilde{c}_x^{\alpha}(t) = \min\{ \text{CVaR}_{\alpha}(\tilde{F}_{x,t}), h_x\}$ according to Algorithm \ref{alg:cvarucb}.
We define the ``good event" $G_x$ as in \cite{tamkin2019distributionally}:
\begin{align*}
    G_x = \{c_1^{\alpha} < \min_{t\in[n]} \Tilde{c}_1^{\alpha}(t)\} \cap \{\Tilde{c}_x^{\alpha}(u_x)< c_1^{\alpha}\},
\end{align*}
where $u_x\in[n]$ is a constant to be chosen later. The ``good event" captures the case when the optimal arm is never underestimated. We expect that when the event $G_x$ occurs, sub-optimal arms will not be pulled frequently and the complement event $G_x^c$ occurs with low probability. We can rewrite $\mathbb{E}[T_x(n)]$ by conditioning on the event $G_x$, as 
\begin{align}\label{eq:tndecompo}
    \mathbb{E}[T_x(n)]=\mathbb{E}[\mathbb{I}\{G_x\}T_x(n)]+\mathbb{E}[\mathbb{I}\{G_x^c\}T_x(n)].
\end{align}
We first consider the case when $G_x$ occurs. a) Suppose condition $C_2$ holds. If $c_1^{\alpha}< \min_{t\in[n]} \Tilde{c}_1^{\alpha}(t)$, i.e., if the optimal arm is never underestimated, then the arm $x$ will not be pulled since $\Tilde{c}_x^{\alpha}(t)\leq h_x<c_1^{\alpha}$, where $\Tilde{c}_x^{\alpha}(t)\leq h_x$ holds according to the definition of $\Tilde{c}_x^{\alpha}(t).$ Thus, when the event $G_x$ occurs, arm $x$ will not be pulled and $\mathbb{E}[\mathbb{I}\{G_x\}T_x(n)]=0$. b) Suppose condition $C_3$ holds. One can prove by contradiction that $T_x(n)\leq u_x$, which follows from the proof of Theorem A.6 in \cite{tamkin2019distributionally}. %

We now consider the case when $G_x^c$ occurs, where  $G_x^c = \{c_1^{\alpha} \geq \min_{t\in[n]} \Tilde{c}_1^{\alpha}(t)\} \cup \{\Tilde{c}_x^{\alpha}(u_x)\geq c_1^{\alpha}\}$. We need to show that $G_x^c$ occurs with low probability. a) Suppose condition $C_2$ holds. We have $P(\{\Tilde{c}_x^{\alpha}(u_x)\geq c_1^{\alpha}\})=0$ due to the fact $\Tilde{c}_x^{\alpha}(t)\leq h_x$ for all $t\in[n]$. According to Theorem A.6 in \cite{tamkin2019distributionally}, we have $P(\{c_1^{\alpha} \geq \min_{t\in[n]} \Tilde{c}_1^{\alpha}(t)\})\leq \frac{1}{n}$. Since $T_x(n)\leq n$, we obtain that $\mathbb{E}[\mathbb{I}\{G_x^c\}T_x(n)]\leq 0+\frac{1}{n}n=1.$ b) 
Suppose condition $C_3$ holds. The value of $\mathbb{E}[\mathbb{I}\{G_x^c\}T_x(n)]$ reduces to that in \cite{tamkin2019distributionally} without causal bounds since the condition $C_3$ does not provide additional information on the values of $\tilde{c}_x^{\alpha}$ to decrease the probability of the event $G_x^c$. Thus, we have  $\mathbb{E}[\mathbb{I}\{G_x^c\}T_x(n)]\leq \frac{n+1}{n}$ as shown in Theorem A.6 in \cite{tamkin2019distributionally}.

In summary, under conditions $C_1$ and $C_2$, the expected number of times that Algorithm \ref{alg:cvarucb} pulls a sub-optimal arm is upper bounded by 0 and 1, respectively. Under condition $C_3$, causal bounds do not provide additional information to reduce the expected number of pulls of sub-optimal arms. In this case, the expected number of pulls matches the upper bounds in \cite{tamkin2019distributionally}. The proof is complete. 
\end{proof}
%
%

Theorem \ref{theorem3} provides conditions under which causal bounds can help decrease the number of pulls of sub-optimal arms. By multiplying the expected number of pulls of sub-optimal arms with their corresponding sub-optimality gaps in \eqref{eq:cvarregret_decompo}, it is straightforward to show that Algorithm \ref{alg:cvarucb} achieves lower regret compared to \cite{tamkin2019distributionally}. This is because if there exists an arm that satisfies condition $C_1$ or $C_2$, then the CVaR regret of Algorithm \ref{alg:cvarucb} is lower than that of \cite{tamkin2019distributionally}.

\section{Numerical Experiments}\label{sec:4_results}
\begin{figure}[t]
	\centering
	\subfigure[\footnotesize Cumulative CVaR-regret of our method (green) and CVaR-UCB \cite{tamkin2019distributionally} (blue). The solid lines and shades are averages and standard deviations over 15 runs.]{\label{subfig:cvar_regret}\includegraphics[scale=0.5]{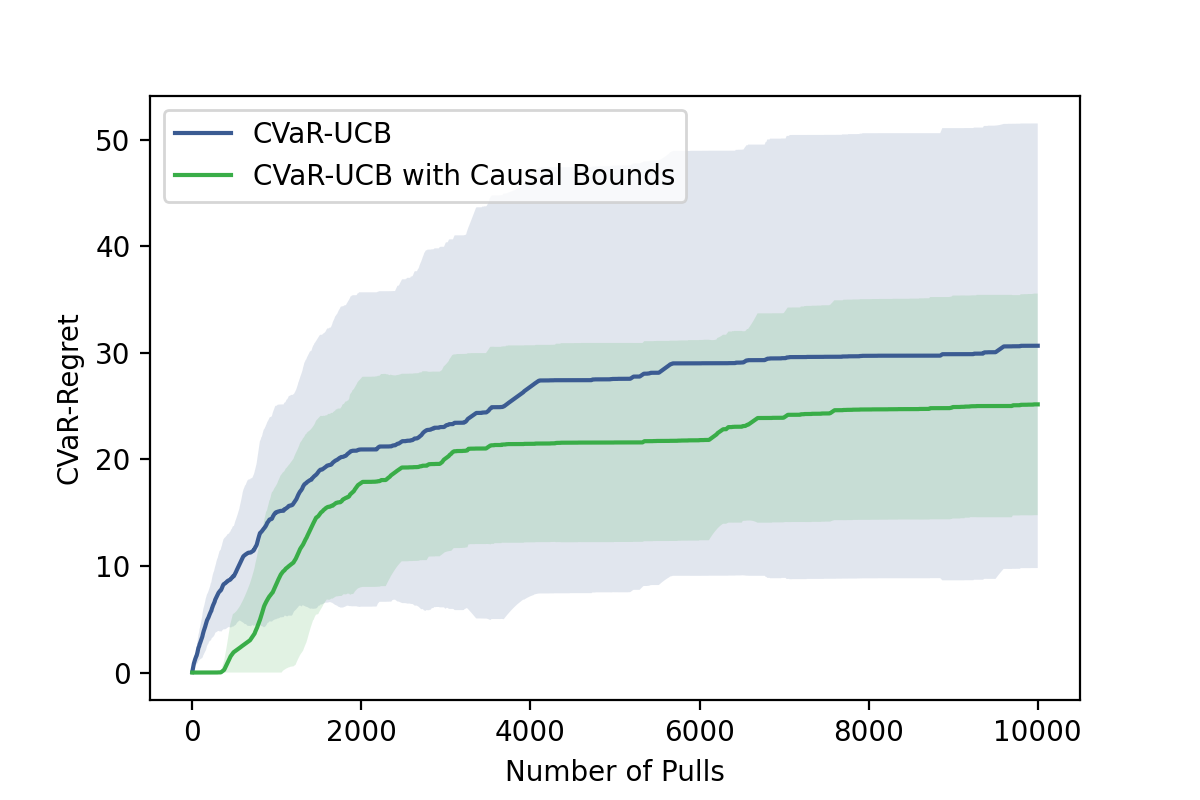}} \quad
	\subfigure[\footnotesize Cumulative regret of our method (green) and standard UCB \cite{auer2002finite} (red). The solid lines and shades are averages and standard deviations over 15 runs.]{\label{subfig:mean_regret}\includegraphics[scale=0.5]{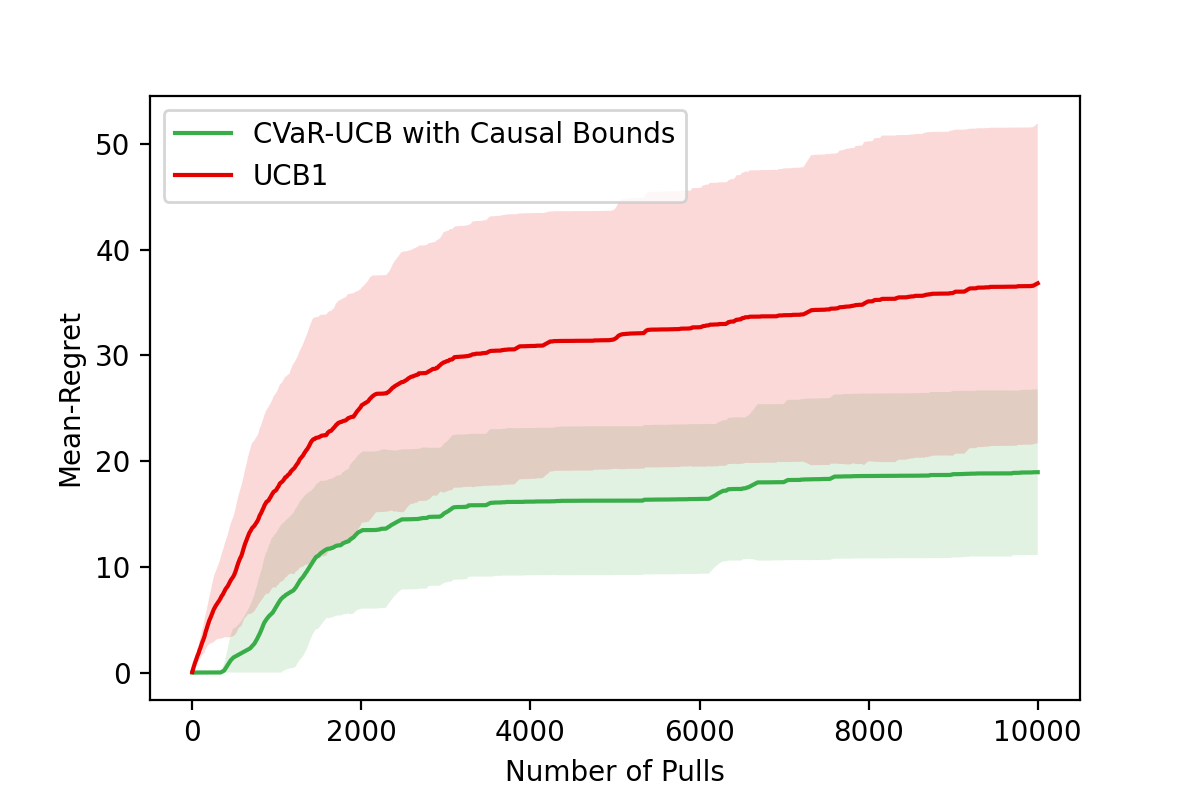}}
	\caption{\small Cumulative regret comparisons}
	\label{figure:results}
\end{figure}

\begin{table}[t]
\caption{Probability tables of generating the observational data by the expert and CVaR causal bounds. (\textit{a}) the joint distribution of $P(x,c)$; (\textit{b}) the probabilities of being effective for every contexts and strategies; (\textit{c}) causal bounds on $\text{CVaR}(Y|do(X))$ to the learner.
\label{tab:probability table}}
\centering
\subtable[$P(X,C)$]{
\begin{tabular}{|c |c|c| }
 \hline
 ~    & $C=1$ & $C=0$ \\ 
 \hline
 $X=1$    & 0.2 & 0.7 \\ 
 \hline
 $X=0$ & 0.8 & 0.3 \\
 \hline
\end{tabular}} 
\subtable[$P(Y=1|X,C$)]{
\begin{tabular}{|c |c|c| } 
 \hline
 ~    & $C=1$ & $C=0$ \\ 
 \hline
 $X=1$    & 0.1 & 0.55\\ 
 \hline
 $X=0$ & 0.3 & 0.45\\
 \hline
\end{tabular}}
\subtable[$\text{CVaR}(Y|do(X))$]{
\begin{tabular}{|c|c|c| } 
 \hline
 ~    & $\alpha=0.75$ & $\text{CVaR}$ \\ 
 \hline
 $X=0$  &[0,0.4] & 0.243\\ 
 \hline
 $X=1$   & [0.29,0.45] & 0.328\\
 \hline
\end{tabular}}
\vspace{-6mm}
\end{table}

Consider an emotion regulation (ER) intervention design problem for people with high Social Interaction Anxiety Scale (SIAS) \cite{mattick1998development}, who are experiencing moderate to severe social anxiety symptoms and are seeking for rapid and adaptive personalized ER intervention to relieve stress and anxiety. Specifically, we choose Seeking advice/comfort from others (S1) and Accepting thoughts/feelings (S2) as two strategies to help manage people's emotion as in \cite{ameko2020offline}. Note that the first strategy S1 is behavioral while S2 is cognitive since it involves a change in one's thinking. \cite{ameko2020offline} concludes that a user's current state of movement, (e.g., being stationary versus moving) can help to determine which ER strategies would regulate his/her emotions most effectively. However, in some mobile health devices there is no activity detection function due to limited sensors; further, people may not carry the devices all the time or intentionally disable the movement detection due to privacy concerns or battery life. As a result, a person's movement information is an unobserved confounder under these circumstances. Nevertheless, data collected from devices that can detect movement can help those devices without such function using the method as proposed Section in \ref{sec:2_method}. We generate a synthetic data to demonstrate this example as follows: we use $C=1$ to indicate that the person is moving and set $P(C=1)=0.12$ to generate the contexts. 
We assume a binary variable $X$ capturing whether S1 is recommended, i.e., $X=1$ if S1 is selected and $X=0$ if S2 is selected; and a binary variable $Y$ capturing whether the person's self-reporting evaluations on the selected ER intervention suggestion is effective or not.. As we assume higher reward is better, we set $Y=1$ if the ER strategy is effective. As indicated by \cite{ameko2020offline}, Seeking advice/comfort from others is more effective for people that are stationary than moving. Thus, when $C=0$, the strategy S1 is selected more often in the expert policy.   
%
The overall context-dependent policy is summarized in Table \ref{tab:probability table}(a). The outcomes of the recommendation being effective ($Y=1$) are generated according to Table \ref{tab:probability table}(b). The observational data containing recommendations (the mobile health app suggestion, S1 or S2) and outcomes (user report of effectiveness) but excluding contextual information (movement status) is then transferred to the learner (the mobile health recommendation system). 
We first apply Theorem \ref{thm:1} to calculate causal bounds on $P(Y|do(X))$.
Then, using Theorem \ref{theorem2}, we can obtain the CVaR causal bounds for a given level of risk $\alpha$; Table \ref{tab:probability table} (c) shows the CVaR causal bounds for $\alpha=0.75$ and the true CVaR value for $\alpha=0.75$. We use Gurobi \cite{gurobi} to solve all the linear and mixed-integer programming problems.
We select $\alpha=0.75$ for our numerical experiments. Specifically, we compare our causal bound constrained CVaR-UCB with CVaR-UCB \cite{tamkin2019distributionally}  using CVaR-regret as a performance measure.

The results in Figure \ref{figure:results}(a) show that the CVaR-regret of our method is lower than the one without causal bounds, e.g., mobile health users wearing the devices without movement detection benefits from the users with advanced devices by avoiding recommendations with high risk. In addition, causal bounds help to reduce the variance. 
We further compare our method with the standard UCB algorithm \cite{auer2002finite} using mean regret as a performance measure to determine whether our proposed risk-averse method can outperform risk-neutral methods using risk-neutral criterion. We observe that, in the two-arm case, our method generates a lower regret and variance compared to the UCB algorithm, as shown in Figure \ref{figure:results}(b). This is because the sub-optimality gap in \eqref{eq:cvarregret_decompo} for the CVaR criterion is larger than the gap for the mean criterion. The larger sub-optimality gap for the CVaR criterion makes the best arm identification problem easier.   
\section{Conclusion}\label{sec:5_conc}

In this work, we proposed a transfer learning method for risk-averse MAB that can handle UCs. Specifically,
 we formulated a mixed-integer linear program (MIP)  that utilizes the observational data to calculate causal bounds on CVaR values. We then transferred these CVaR causal bounds to the learner and proposed a causal bound constrained UCB algorithm to reduce the variance of online learning. We provided a regret analysis and showed that our method can achieve zero or constant regret using causal bounds under certain conditions. To illustrate our proposed method, we simulated a mobile health emotion regulation recommender system and demonstrated that interventions can be chosen more appropriately and with lower risk using our method.
\addtolength{\textheight}{-11cm}   

\bibliography{ref.bib}
\bibliographystyle{IEEEtran}

\end{document}